\documentclass[runningheads]{llncs}
\usepackage[T1]{fontenc}
\usepackage{graphicx}
\usepackage{color}
\usepackage{amssymb,mathtools}
\usepackage{mathrsfs}
\usepackage{bm}
\usepackage{enumitem}
\usepackage{hyperref}
\usepackage{microtype}
\usepackage{wrapfig}
\newcommand{\R}{\mathbb{R}}
\newcommand{\E}{\mathbb{E}}
\newcommand{\Cov}{\mathrm{Cov}}

\newcommand{\grad}{\mathrm{grad}}
\newcommand{\diver}{\mathrm{div}}
\newcommand{\Id}{\mathrm{Id}}
\newcommand{\cP}{\mathcal{P}}
\newcommand{\T}{\mathbb{T}}
\newcommand{\reg}{\mathrm{reg}}
\newcommand{\tr}{\mathrm{tr}}
\newcommand{\normg}[1]{\lVert #1 \rVert_g}
\newcommand{\norm}[1]{\lVert #1 \rVert}

\newtheorem{assumption}[theorem]{Assumption}

\begin{document}
\title{A Link between Shock-wave Theory and Symmetry-reduced Stochastic Gradient Descent for Artificial Neural Networks}
%
%
\author{Taiki Miyagawa}
%
\authorrunning{Taiki Miyagawa}
%
\institute{
    NEC Corporation
    \email{miyagawataik@nec.com}
}
\maketitle              
\begin{abstract}
We develop a mathematically explicit link between shock-wave theory and the symmetry-quotiented learning dynamics of stochastic gradient descent, drawing on differential geometry, Lie group theory, and fluid mechanics.
Specifically, after \textit{quotienting parameter symmetries} and applying \textit{local-entropy coarse-graining}, the effective dynamics satisfy a viscous Hamilton--Jacobi equation on the quotient manifold.
Moreover, under the assumption that the raw parameter dynamics can be summarized by a gradient field on the quotiented space, the gradient of the coarse-grained loss function obeys a Burgers-type equation, and shock formation can be established rigorously.
We apply our theory to multilayer perceptrons, convolutional neural networks, Transformers, and mean-field networks, and show that they obey the Hamilton--Jacobi or Burgers-type equations.
We conjecture that this framework also yields practical diagnostics for deep learning. In architectures such as Transformers, raw parameter norms are often distorted by symmetry redundancy and may therefore be misleading, whereas symmetry-corrected quotient observables provide a principled basis for monitoring, forecasting, and controlling training-phase transitions.
\keywords{Shock waves \and Stochastic gradient descent \and Hamilton--Jacobi equation \and Burgers equation.}
\end{abstract}

\section{Introduction}

Combining the following insights, we propose a correspondence between shock wave theory and symmetry-quotiented learning dynamics of stochastic gradient descent (SGD).
Shock waves in fluid mechanics are governed by nonlinear transport, loss of classical regularity, and weak-solution selection by entropy conditions. Deep learning, by contrast, is usually formulated as a high-dimensional stochastic optimization problem. 
Several established mathematical facts suggest a principled bridge. 
First, positively homogeneous neural networks such as ReLU networks possess positive rescalings and permutations, so physically meaningful observables often live on \textit{quotient spaces} rather than in raw parameter coordinates \cite{Neyshabur2015,Rangamani2019}. 
Second, discrete-time SGD admits continuous-time approximations in the form of stochastic modified equations and stochastic modified flows \cite{LiTaiE2019,GessKassingKonarovskyi2024}. 
Third, local-entropy relaxations of nonconvex losses are governed by viscous Hamilton--Jacobi equations \cite{ChaudhariObermanOsherSoattoCarlier2018}. 
Fourth, in wide-network limits, SGD induces diffusion equations rather than simple finite-dimensional ordinary differential equations \cite{MeiMontanariNguyen2018,SirignanoSpiliopoulos2022}.

The purpose of this paper is to connect these ingredients in a single rigorous architecture.\footnote{
    Our position is deliberately conservative: We do not claim that generic neural-network parameters satisfy Burgers' equation.
}
We claim that \textit{symmetry quotient plus local-entropy coarse-graining} naturally yields a viscous Hamilton--Jacobi equation on a \textit{quotient space}. 
Furthermore, we prove that if a gradient field on the quotient space summarizes the dynamics on the raw parameter space (referred to in this paper as the closedness assumption for a one-dimensional collective coordinate), then, the quotiented gradient field obeys a Burgers-type equation, with shock formation time controlled by the negative curvature of the coarse-grained loss function.
This yields a precise mathematical reinterpretation of abrupt training-regime changes: in the quotient description, they appear as shock-type singularities or viscous shock layers in the coarse-grained average gradient. 

Beyond its mathematical correspondence with Hamilton--Jacobi and Burgers-type equations, we conjecture that the present framework has a practical meaning for modern deep learning systems. The main point is that the theory identifies which variables should be monitored, which quantities should be interpreted as early-warning signals of regime change, and which hyperparameters act as control knobs for smoothing or sharpening such transitions. 


\section{Preliminaries} \label{sec:Preliminaries}
\subsection{Definition}
\begin{wrapfigure}{r}{0.5\textwidth}
    \begin{center}
    \includegraphics[width=0.48\textwidth]{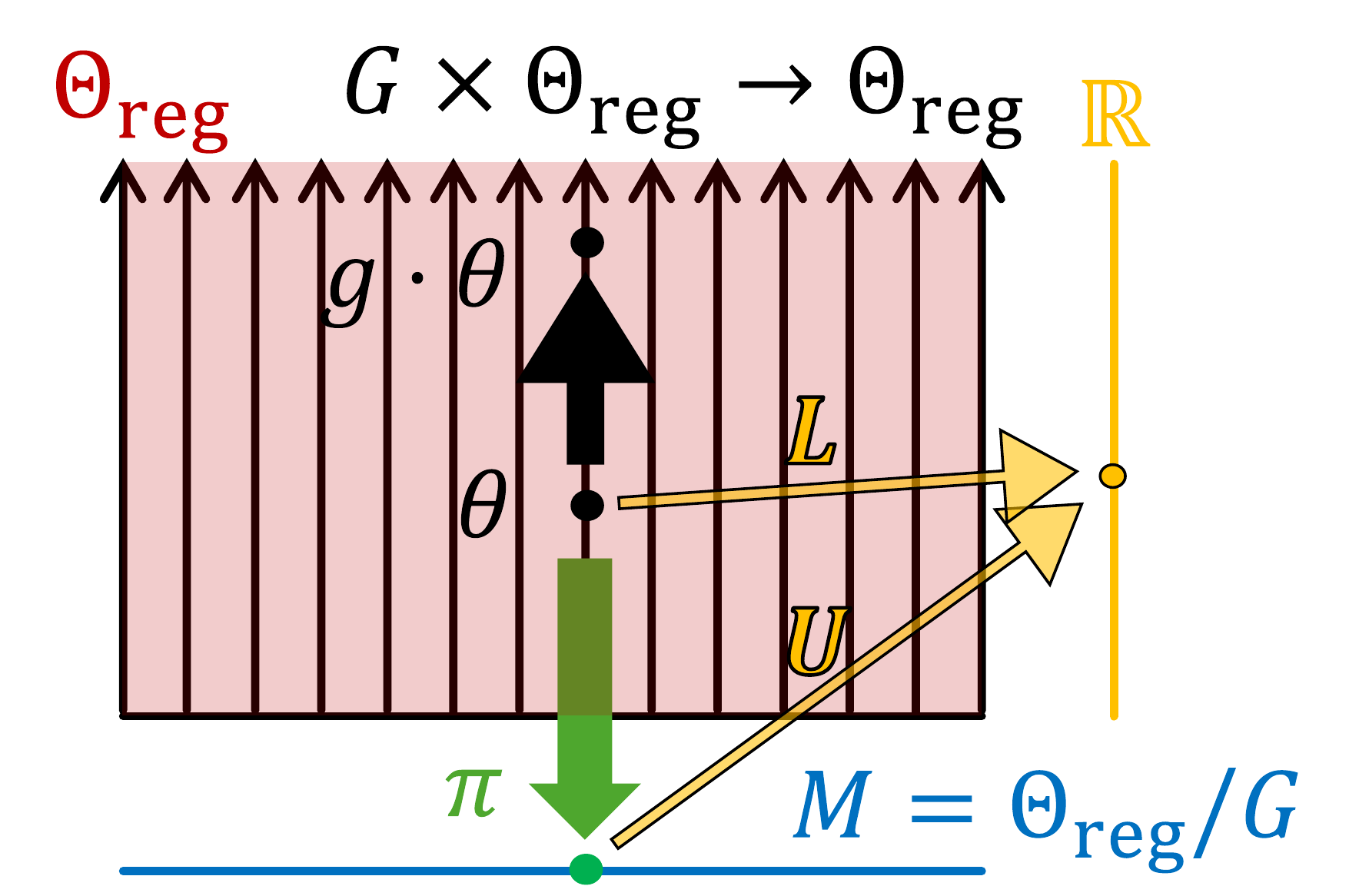}
    \end{center}
    \caption{\textbf{Notation.}} 
    \label{fig:Notation}
\end{wrapfigure}
Let $\Theta \subset \R^{d_\Theta}$ be a smooth parameter manifold, and let a Lie group or finite group $G$ act smoothly on $\Theta$. We assume that there is an open regular stratum $\Theta_{\reg} \subset \Theta$ on which the action is free and proper. 
Then, the quotient
$M := \Theta_{\reg}/G$
forms a smooth manifold, and the quotient map $\pi : \Theta_{\reg} \to M$ is a smooth submersion. 
In the finite-group case, the global quotient may be an orbifold, but on each principal stratum, the local manifold picture is valid.
Let $L : \Theta_{\reg} \to \R$ be a smooth empirical loss satisfying
$L(g \cdot \theta) = L(\theta)$
for all $g \in G$ and $\theta \in \Theta_{\reg}$. 
Then, $L$ descends to a smooth function, \textit{effective potential }, $U : M \to \R$ such that $L = U \circ \pi$.
We consider the stochastic iteration
$\theta_{n+1} = \theta_n - \eta \bigl(\nabla L(\theta_n) + M_{n+1}\bigr),$
where $\eta > 0$ is the learning rate and $(M_{n+1})_{n \ge 0}$ is a martingale-difference sequence adapted to a filtration $(\mathcal{F}_n)_{n \ge 0}$, that is,
$\E[M_{n+1} \mid \mathcal{F}_n] = 0.$
We assume throughout that on compact subsets of $\Theta_{\reg}$,
$\E\bigl[\norm{M_{n+1}}^3 \mid \mathcal{F}_n\bigr] \le C$
for some local constant $C$.
We will repeatedly use the standard Hopf--Cole transform\footnote{
    It is a change of variables that transforms a special type of parabolic partial differential equations (PDEs) with a quadratic nonlinearity into a linear heat equation.
} for viscous Burgers and viscous Hamilton--Jacobi equations, together with classical characteristic theory for inviscid Burgers; standard references include Evans and LeVeque \cite{Evans2010,LeVeque1992}.

\subsection{Quotient Reduction}
We first show that, under a \textit{local projectability assumption}, the discrete-time SGD recursion descends in a quotient chart to a closed stochastic recursion whose drift and conditional covariance depend only on the quotient state. This is the basic reduction that allows us to replace the raw parameter dynamics by an effective dynamics on the symmetry quotient space.

\begin{assumption}[Local projectability of drift and covariance]\label{ass:projectability}
    Let $\chi : U \subset M \to \R^m$ be a smooth chart and set $\Phi := \chi \circ \pi$. Assume the trajectory remains in $\pi^{-1}(U)$ almost surely, and that $\Phi$ is $C^3$ on the relevant compact subset. Assume further that there exist locally bounded functions
    $b : \chi(U) \to \R^m$ and $A : \chi(U) \to \R^{m \times m}$
    such that, almost surely,
    $D\Phi(\theta_n)\nabla L(\theta_n) = b(Y_n)$ (referred to as \textit{drift} in this paper), $Y_n := \Phi(\theta_n),$
    and
    $\Cov\bigl(D\Phi(\theta_n)M_{n+1} \mid \mathcal{F}_n\bigr) = A(Y_n).$\footnote{
        Here, $D\Phi(\theta)$ denotes the derivative (Jacobian) of the quotient-chart map $\Phi$ at $\theta$, which linearly maps infinitesimal parameter-space displacements to quotient-coordinate displacements. Accordingly, $D\Phi(\theta_n)\nabla L(\theta_n)$ is the gradient projected into the quotient coordinates, and $D\Phi(\theta_n)M_{n+1}$ is the noise projected into the quotient coordinates.
    }
\end{assumption}
Assumption~\ref{ass:projectability} requires that, in a local quotient chart, the projected gradient drift\footnote{
    Drift is the deterministic mean component of a stochastic update, that is, the average direction of motion after averaging out the random fluctuations.
} $b(Y_n)$ and the conditional covariance of the projected martingale noise depend only on the quotient state $Y_n$. 
Thus, the stochastic evolution closes at the level of the quotient variables: different representatives of the same symmetry orbit induce the same effective first- and second-order dynamics after projection.\footnote{
    ReLU networks are a natural class of models in which to consider Assumption 1, because they do possess the relevant symmetries. However, they do not satisfy Assumption 1 automatically in full generality. On a regular stratum, they satisfy it locally if the projected drift and covariance close as functions only of the quotient.
}

\begin{theorem}[Local quotient reduction for discrete-time SGD]\label{thm:quotient_sgd}
Under Assumption~\ref{ass:projectability}, there exist random variables $\Xi_{n+1}$ and $R_n$ such that
$Y_{n+1} = Y_n - \eta b(Y_n) + \eta \Xi_{n+1} + \eta^2 R_n,$
with
$\E[\Xi_{n+1} \mid \mathcal{F}_n] = 0$, $ \Cov(\Xi_{n+1} \mid \mathcal{F}_n) = A(Y_n),$
and $R_n$ locally bounded in conditional expectation.
\end{theorem}

\begin{proof}
Write
$\Delta_n := \theta_{n+1} - \theta_n = -\eta \bigl(\nabla L(\theta_n) + M_{n+1}\bigr).$
Taylor's theorem with integral remainder gives
$\Phi(\theta_n + \Delta_n) = \Phi(\theta_n) + D\Phi(\theta_n)[\Delta_n] + \frac{1}{2} D^2\Phi(\theta_n)[\Delta_n,\Delta_n] + \mathcal{R}_{n+1},$
where, for some random point on the segment joining $\theta_n$ and $\theta_n + \Delta_n$,
$\norm{\mathcal{R}_{n+1}} \le C \norm{\Delta_n}^3.$
Since $\E[\norm{M_{n+1}}^3 \mid \mathcal{F}_n] \le C$ locally, we have $\E[\norm{\Delta_n}^3 \mid \mathcal{F}_n] = O(\eta^3)$ and hence
$\E[\norm{\mathcal{R}_{n+1}} \mid \mathcal{F}_n] = O(\eta^3).$
Moreover,
$D^2\Phi(\theta_n)[\Delta_n,\Delta_n] = O_{L^1}(\eta^2)$
locally, because $\Delta_n = O(\eta)$ in conditional $L^2$.

Therefore,
\begin{equation*}
\begin{aligned}
Y_{n+1} - Y_n
&= D\Phi(\theta_n)[\Delta_n] + O_{L^1}(\eta^2) \\
&= -\eta D\Phi(\theta_n)[\nabla L(\theta_n)] - \eta D\Phi(\theta_n)[M_{n+1}] + O_{L^1}(\eta^2).
\end{aligned}
\end{equation*}
Define
$\Xi_{n+1} := - D\Phi(\theta_n)[M_{n+1}].$
Because $D\Phi(\theta_n)$ is $\mathcal{F}_n$-measurable and $\E[M_{n+1} \mid \mathcal{F}_n] = 0$, we obtain
$\E[\Xi_{n+1} \mid \mathcal{F}_n] = 0.$
By Assumption~\ref{ass:projectability},
$D\Phi(\theta_n)[\nabla L(\theta_n)] = b(Y_n)$
a.s., and
$\Cov(\Xi_{n+1} \mid \mathcal{F}_n) = A(Y_n).$
Collecting the $O_{L^1}(\eta^2)$ terms into $\eta^2 R_n$ yields the claimed recursion. \qed
\end{proof}

If the chart is fixed and the coefficients are sufficiently regular, then the recursion above admits, with $\eta n \fallingdotseq t$, the standard weak continuous-time approximation
$dY_t = - b(Y_t) \, dt + \sqrt{\eta} \, \sigma(Y_t) \, dB_t$ and $\sigma \sigma^{\top} = A,$
which is the symmetry-reduced analogue of stochastic modified equations and stochastic modified flows for SGD \cite{LiTaiE2019,GessKassingKonarovskyi2024}.
The benefit of the theorem is fundamental for the rest of the paper.

\subsection{Coarse-graining, Quotient Local Entropy}
Let $(M,g)$ be a complete Riemannian manifold with the Laplace--Beltrami operator $\Delta_M$.\footnote{
    $\Delta_M$ denotes the Laplace–Beltrami operator on the quotient Riemannian manifold $(M,g)$, i.e. the geometric Laplacian $\Delta_M f=\operatorname{div}_g(\operatorname{grad} f)$ governing diffusion on $M$.
} 
Let $U : M \to \R$ be a smooth \textit{effective potential} on the quotient, or equivalently a regularized effective loss function. 
Define the heat semigroup
$P_t := e^{\frac{t}{2} \Delta_M}.$
For viscosity parameter $\nu > 0$ and coarse-graining scale $\tau \ge 0$, define the \textit{local-entropy regularization} $u^\nu(\tau,q) := -\nu \log\bigl(P_{\nu \tau} e^{-U/\nu}(q)\bigr)$, where $q \in M$.
\textit{Coarse-graining} refers to replacing the raw parameter dynamics by an effective dynamics obtained after symmetry quotienting and local-entropy smoothing on the symmetry quotient.
The discrete SGD time $n$ and the continuous heat time $t$ are identified at the scaling level through $ \eta n \fallingdotseq t = \nu \tau$, where $\nu$ denotes the viscosity, equivalently coarse-graining scale, and $\tau$ is a continuous normalized coarse-graining parameter.

\section{Hamilton--Jacobi Equation} \label{sec:HJE}
We propose that symmetry quotient plus local-entropy coarse-graining yields a viscous Hamilton--Jacobi equation on a quotient space:

\begin{theorem}[Quotient Hamilton--Jacobi equation]\label{thm:HJ}
Assume $U \in C^2(M)$, assume $P_t e^{-U/\nu}$ is strictly positive for $t \ge 0$, and assume the function
$w(\tau,q) := P_{\nu \tau} e^{-U/\nu}(q)$
belongs to $C^{1,2}((0,\infty) \times M)$ and solves the heat equation pointwise:
$\partial_\tau w = \frac{\nu}{2} \Delta_M w$, and $w(0,q) = e^{-U(q)/\nu}.$\footnote{
    If $P_t$ is the heat semigroup generated by $\frac12\Delta_M$ and the required regularity holds, then, $w(\tau,q)=P_{\nu\tau}e^{-U/\nu}(q)$ solves $\partial_\tau w=\frac{\nu}{2}\Delta_M w$ automatically.
}  
Then, the function $u^\nu$ defined by
$u^\nu(\tau,q) = -\nu \log w(\tau,q)$
solves
$\partial_\tau u^\nu + \frac{1}{2} \normg{\grad u^\nu}^2 = \frac{\nu}{2} \Delta_M u^\nu$, $u^\nu(0,q) = U(q).$
\end{theorem}

\begin{proof}
Because $w > 0$, the logarithm is well defined. By differentiation,
$\partial_\tau u^\nu = -\nu \frac{\partial_\tau w}{w} = - \frac{\nu^2}{2} \frac{\Delta_M w}{w}.$
Also,
$\grad u^\nu = -\nu \frac{\grad w}{w}, \qquad \normg{\grad u^\nu}^2 = \nu^2 \frac{\normg{\grad w}^2}{w^2}.$
Using the identity
$\Delta_M(\log w) = \frac{\Delta_M w}{w} - \frac{\normg{\grad w}^2}{w^2},$
we obtain
$\Delta_M u^\nu = -\nu \frac{\Delta_M w}{w} + \nu \frac{\normg{\grad w}^2}{w^2}.$
Therefore,
$\partial_\tau u^\nu + \frac{1}{2} \normg{\grad u^\nu}^2 - \frac{\nu}{2} \Delta_M u^\nu = 0.$
The initial condition follows from $P_0 = \Id$. \qed 
\end{proof}

Thm.~\ref{thm:HJ} shows that, \textit{after symmetry quotienting, quotient local entropy is not merely a heuristic smoothing of the loss, but exactly the viscous Hamilton--Jacobi evolution of the effective potential $U$ on the quotient manifold $M$.} 
Crucially, the nonlinear term $\frac{1}{2}|\operatorname{grad} u^\nu|g^2$ is the mechanism that drives characteristic steepening: in the small-viscosity regime, it tends to sharpen gradients and compress information into increasingly narrow transition layers. Thus, Thm.~\ref{thm:HJ} already contains the mathematical precursor of shock formation. The viscosity term $\frac{\nu}{2}\Delta_M u^\nu$ does not remove this mechanism; rather, it regularizes it, replacing discontinuous shocks by thin viscous shock layers. This is precisely why the quotient Hamilton–Jacobi equation is the correct entry point to the shock-wave interpretation of symmetry-reduced SGD.

The benefit of the theorem is threefold. First, it gives a mathematically controlled effective landscape on which sharp geometric features are regularized by a viscosity parameter $\nu$. Second, it makes available the standard Hamilton--Jacobi toolbox, including Hopf--Cole linearization, semigroup methods, comparison arguments, and small-viscosity asymptotics, for the analysis of symmetry-quotiented learning dynamics. Third, it provides the precise entry point to the Burgers-type description developed later: once a one-dimensional collective coordinate closes the reduced dynamics, the gradient of $u^\nu$ inherits a viscous transport structure whose steepening and possible shock-layer formation can be analyzed quantitatively.

\subsection{Numerical illustration: Hopf--Cole shock layer in a quotient ReLU model}
\label{subsec:relu_hopf_cole_experiment}

We include a small numerical experiment to verify that the quotient
Hopf--Cole quantity introduced above produces a shock-like transition layer in
an explicitly symmetry-reduced ReLU model.  The purpose of the experiment is
not to optimize predictive performance on a benchmark dataset, but to test the
geometric mechanism in a setting where the quotient coordinates and the
quotient Laplace--Beltrami operator can be computed directly.

\paragraph{Model and quotient coordinates.}
We use a one-hidden-layer ReLU network
$    f(x)
    =
    \sum_{j=1}^{m} a_j \sigma(u_j^\top \widetilde{x}) + c,
    \widetilde{x}=(x,1),
    \sigma(z)=\max\{z,0\},
$
with input dimension one and hidden width \(m=2\).  The ReLU positive
rescaling symmetry
$
    (a_j,u_j) \sim (a_j/r_j,r_j u_j),
    r_j>0,
$
is fixed by the balanced representative
$
    |a_j|=\|u_j\|.
$
Equivalently, the quotient coordinates are
$
    \gamma_j = a_j\|u_j\|,
    s_j = \frac{u_j}{\|u_j\|}\in S^1,
    c\in\mathbb{R}.
$
In these coordinates the quotient network is written as
$
    f_Q(x;\gamma,s,c)
    =
    \sum_{j=1}^{m}\gamma_j\sigma(s_j^\top \widetilde{x})+c.
$
For the binary classification loss \(U\), we use the empirical binary
cross-entropy
\begin{equation}
    U(\gamma,s,c)
    =
    \frac{1}{N}\sum_{i=1}^{N}
    \ell_{\mathrm{BCE}}\bigl(f_Q(x_i;\gamma,s,c),y_i\bigr).
\end{equation}

\paragraph{Experimental setting.}
The dataset is a deterministic one-dimensional step classification problem.
We sample \(N=96\) points in \([-2,2]\), add small Gaussian jitter with
standard deviation \(0.035\), and assign labels
$y_i = \mathbf{1}_{\{x_i>0\}}.$
The balanced gauge-fixed ReLU model is trained for \(2000\) epochs using SGD
with learning rate \(0.035\).  After training, the final loss is
\(2.8620\times 10^{-2}\), the training accuracy is \(1.0000\), and the maximum
balanced-gauge error is \(2.384\times 10^{-7}\) (the balanced gauge remains fixed).

We then evaluate the quotient Hopf--Cole profile on a one-dimensional slice of
the learned quotient coordinates.  Specifically, the first spherical direction
\(s_1\in S^1\) is rotated by an angle \(\theta\in[-1.35,1.35]\), while all
other quotient coordinates are held fixed.  On this slice we compute
\begin{equation}
    w(\tau,\theta)
    =
    \exp\!\bigl(\nu\tau\Delta_Q\bigr)
    \exp\!\left(-\frac{U(\theta)}{\nu}\right),
    \qquad
    u^\nu(\tau,\theta)
    =
    -\nu\log w(\tau,\theta),
\end{equation}
where \(\Delta_Q\) is the quotient Laplace--Beltrami operator associated with
the balanced quotient metric.  Numerically, the heat semigroup is approximated
to first order,
\begin{equation}
    w(\tau,\theta)
    \approx
    \phi^\nu(\theta)
    +
    \nu\tau \Delta_Q\phi^\nu(\theta),
    \qquad
    \phi^\nu(\theta)=\exp\!\left(-\frac{U(\theta)}{\nu}\right).
\end{equation}
We use viscosities
$
    \nu\in\{0.05,0.1,0.5\}
$
and heat-flow times
$
    \tau\in\{0,0.006,0.012,0.05,0.1\}.
$
For each pair \((\nu,\tau)\), we plot \(u^\nu\), its first derivative
\(\partial_\theta u^\nu\), and its second derivative
\(\partial_\theta^2 u^\nu\).  Derivatives are computed by finite differences
on a uniform grid of \(91\) points in \(\theta\).

\begin{figure}[t]
    \centering
    \includegraphics[width=0.9\textwidth]{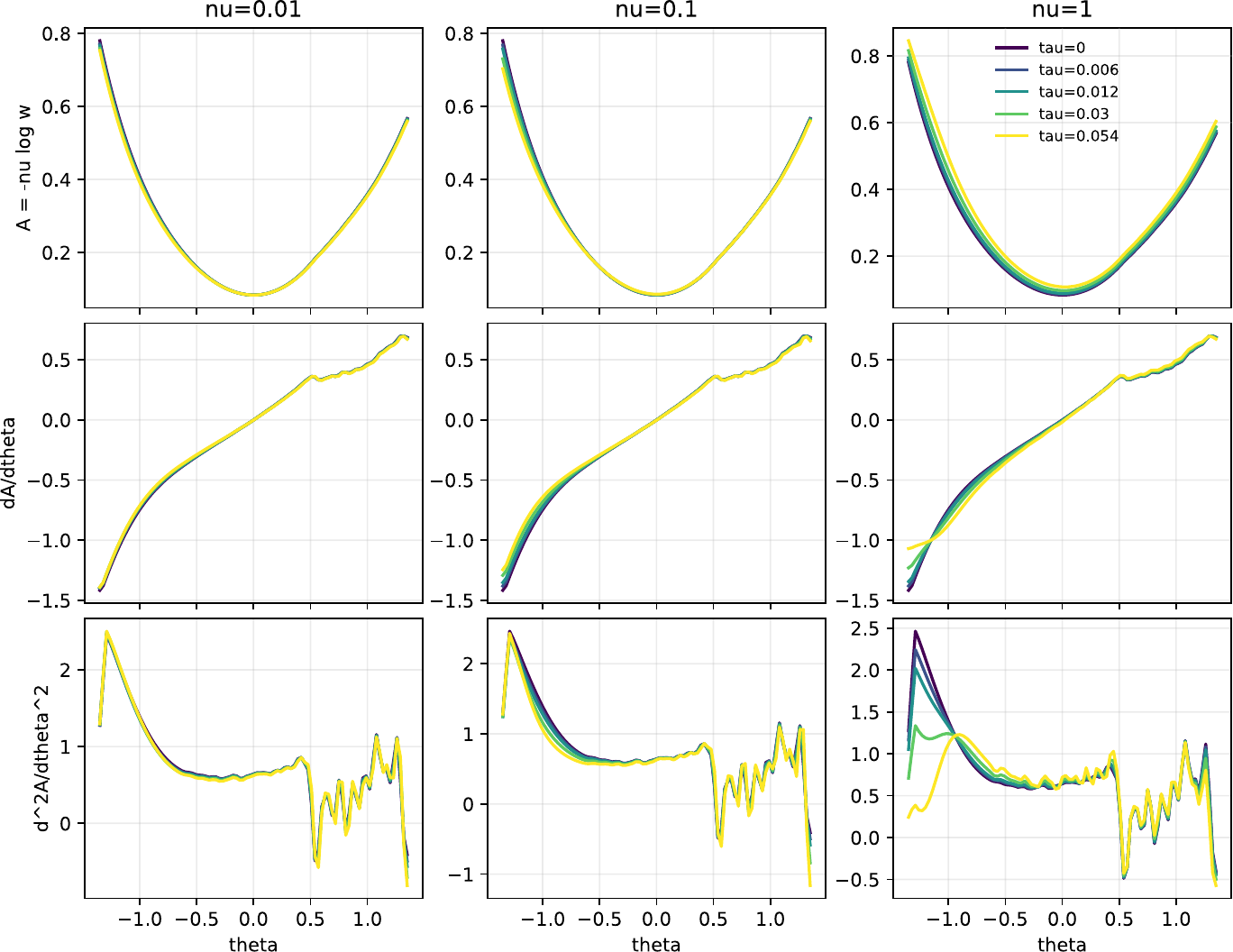}
    \caption{\textbf{Hopf--Cole shock profile in a quotient ReLU model.}
    The columns correspond to \(\nu=0.01,0.1,1.0\), and the curves in each
    panel correspond to increasing heat-flow times
    \(\tau\in\{0,0.006,0.012,0.030,0.054\}\).  The top row shows the effective
    action \(A^\nu=-\nu\log w^\nu\), the middle row shows the quotient-slice
    velocity \(\partial_\theta A^\nu\), and the bottom row shows
    \(\partial_\theta^2 A^\nu\).  Here \(\theta\) is the quotient-slice
    coordinate obtained by rotating the learned first hidden-unit direction
    \(s_1^\ast\in S^1\) as
    \(s_1(\theta)=\cos\theta\,s_1^\ast+\sin\theta\,J s_1^\ast\), with all
    other quotient coordinates fixed.  Smaller viscosity produces a stable,
    sharper transition layer, whereas larger viscosity diffuses the profile.
    The underlying network is a width-\(2\), one-hidden-layer balanced
    gauge-fixed ReLU classifier trained on the one-dimensional task
    \(y=\mathbf{1}_{\{x>0\}}\) with \(96\) samples, plain SGD, learning rate
    \(0.035\), and \(2000\) epochs.}
    \label{fig:relu_hopf_cole_shock_profile}
\end{figure}

\begin{table}[t]
    \centering
    \caption{Maximum absolute second derivative at the final heat-flow time
    \(\tau=0.1\).  The quantity
    \(\max_\theta|\partial_\theta^2 u^\nu|\) measures the sharpness of the
    transition layer in the quotient slice.}
    \label{tab:relu_hopf_cole_curvature}
    \begin{tabular}{c c c}
        \hline
        \(\nu\) & \(\max_\theta|\partial_\theta^2 u^\nu|\) & \(\min_\theta w\) \\
        \hline
        \(0.05\) & \(4.663319\) & \(1.001520\times 10^{-6}\) \\
        \(0.10\) & \(4.412071\) & \(1.456738\times 10^{-3}\) \\
        \(0.50\) & \(3.328405\) & \(2.831267\times 10^{-1}\) \\
        \hline
    \end{tabular}
\end{table}

Figure~\ref{fig:relu_hopf_cole_shock_profile} shows the qualitative behavior
predicted by the quotient Hamilton--Jacobi and Burgers picture.  The local entropy
\(u^\nu=-\nu\log w\) develops a localized steep transition along
the quotient slice, and the corresponding velocity
\(\partial_\theta u^\nu\) rapidly changes over a narrow interval of
\(\theta\).  The second derivative \(\partial_\theta^2 u^\nu\) makes this
transition layer explicit: it concentrates near the location where the
velocity steepens.  As the viscosity increases from \(\nu=0.05\) to
\(\nu=0.5\), the peak curvature decreases from \(4.663319\) to \(3.328405\),
which is consistent with viscous smoothing of a shock layer.

This experiment therefore provides a concrete finite-dimensional illustration
of the theoretical mechanism.  After quotienting the ReLU rescaling symmetry,
the Hopf--Cole transform on the quotient space produces a scalar local entropy
whose gradient exhibits Burgers-type steepening.  The observed layer is
not a discontinuity in raw parameter space; it is a sharp transition in a
symmetry-corrected quotient coordinate, exactly the level at which the theory
predicts shock-like behavior.

\section{Burgers-type Equation} \label{sec:BE} 
We now ask when the quotient Hamilton--Jacobi equation reduces to a scalar transport equation.

\begin{assumption}[One-dimensional closure with isoparametric condition]\label{ass:closure}
Let $\psi : M \to I \subset \R$ be a $C^3$ collective coordinate and let $N \subset M$ be a tubular neighborhood.\footnote{
    A tubular neighborhood of a submanifold is a neighborhood that is diffeomorphic to a neighborhood of the zero section in its normal bundle, so that nearby points are represented by normal displacements from the submanifold.
} Assume:
\begin{enumerate}[label=(\alph*)]
    \item there exists the reduced effective potential $\bar U : I \to \R$ such that $U(q) = \bar U(\psi(q))$ for all $q \in N$;
    \item the coarse-grained solution preserves this dependence on $N$, namely,
$u^\nu(\tau,q) = \bar u^\nu(\tau,\psi(q))$
    for all $(\tau,q) \in [0,T] \times N$;
    \item $\psi$ is unit speed on $N$:
$\normg{\grad \psi} = 1;$
    \item $\Delta_M \psi$ depends only on $\psi$ on $N$, that is, there exists $\kappa \in C^1(I)$ such that
$\Delta_M \psi = \kappa \circ \psi$
    on $N$;
    \item $\bar u^\nu \in C^{1,3}([0,T] \times I)$.
\end{enumerate}
\end{assumption}

\begin{theorem}[Source-corrected one-dimensional reduction]\label{thm:burgers_reduction}
    Under Assumption~\ref{ass:closure}, $\bar u^\nu$ satisfies
    $\partial_\tau \bar u^\nu + \frac{1}{2} (\partial_s \bar u^\nu)^2 = \frac{\nu}{2} \bigl(\partial_{ss} \bar u^\nu + \kappa(s) \partial_s \bar u^\nu\bigr),$
    where $s = \psi(q)$. Consequently, the gradient field
    $v^\nu(\tau,s) := \partial_s \bar u^\nu(\tau,s)$
    satisfies
    $\partial_\tau v^\nu + v^\nu \partial_s v^\nu = \frac{\nu}{2} \bigl(\partial_{ss} v^\nu + \kappa(s) \partial_s v^\nu + \kappa'(s) v^\nu\bigr).$
    If, in addition, $\kappa \equiv 0$, then the equation reduces to the classical viscous Burgers equation
    $\partial_\tau v^\nu + v^\nu \partial_s v^\nu = \frac{\nu}{2} \partial_{ss} v^\nu.$
\end{theorem}

\begin{proof}
    Substitute $u^\nu(\tau,q) = \bar u^\nu(\tau,\psi(q))$ into the Hamilton--Jacobi equation. By the chain rule,
    $\grad u^\nu = (\partial_s \bar u^\nu) \grad \psi$, and $\normg{\grad u^\nu}^2 = (\partial_s \bar u^\nu)^2 \normg{\grad \psi}^2 = (\partial_s \bar u^\nu)^2.$
    For the Laplacian,
    $
    \Delta_M u^\nu
    = \partial_{ss} \bar u^\nu \, \normg{\grad \psi}^2 + \partial_s \bar u^\nu \, \Delta_M \psi \\
    = \partial_{ss} \bar u^\nu + \kappa(\psi(q)) \partial_s \bar u^\nu.
    $
    Since the right-hand side depends on $q$ only through $s = \psi(q)$, the scalar equation follows. Differentiating it w.r.t. $s$ is justified by $\bar u^\nu \in C^{1,3}$ and gives the evolution for $v^\nu$. If $\kappa \equiv 0$, the source terms vanish. \qed
\end{proof}


\begin{figure}[htbp]
    \centering
    \includegraphics[width=\textwidth]{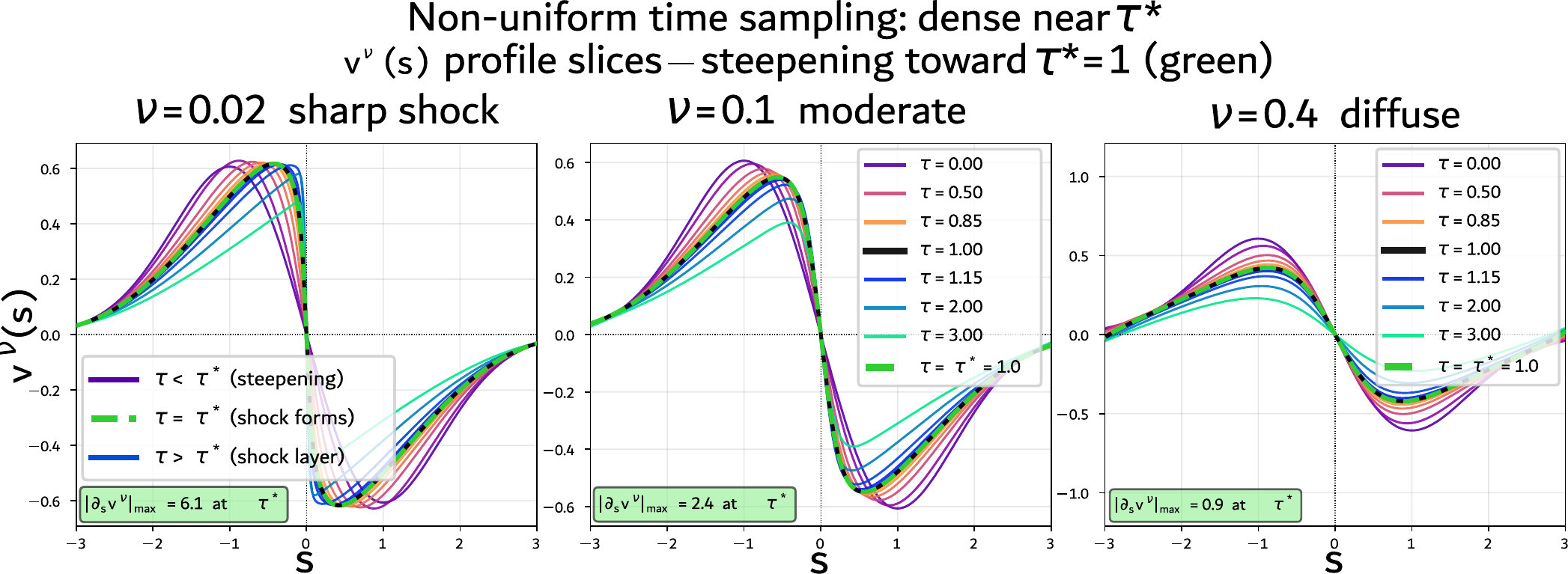}
    \caption{
    \textbf{Shock formation in viscous Burgers equation via Hopf--Cole transform.}
    Each panel shows the velocity field $v^\nu(\tau, s)$, obtained from the Hopf--Cole solution $v^\nu = \partial_s(-\nu \log w)$ with potential $U_0(s) = e^{-s^2/2}$, for three viscosities $\nu \in \{ 0.02, 0.10, 0.40 \}$ (left to right). The initial drift is  $v_0(s) = U'_0 (s) = -se^{-s^2/2}$, whose second derivative satisfies $\inf_s U''_0 (s) = -1$ at $s=0$. Time samples are non-uniformly spaced with higher density near $\tau^*$ to resolve the rapid transition. As $\nu \rightarrow 0$, the shock layer sharpens into a discontinuity, while large $\nu$ smooths the front and suppresses shock formation entirely.
    Note that this is a numerical integration (simulation) of the viscous Burgers equation.
    }
    \label{fig:ShockFormation}
\end{figure}
Thm.~\ref{thm:burgers_reduction} shows that, \textit{once the quotient coarse-grained local-entropy dynamics close through a single collective coordinate $s=\psi(q)$, the symmetry-quotiented
Hamilton--Jacobi equation collapses to a scalar nonlinear evolution, and its gradient field becomes a Burgers-type equation.} Thus, the effective training dynamics admit a genuine transport normal form: the gradient field $v^\nu=\partial_s \bar u^\nu$ evolves by nonlinear self-advection, while diffusion and geometric forcing appear as lower-order corrections.

The benefit of the theorem is that it converts a high-dimensional quotient dynamics on $M$ into a one-dimensional local-entropy dynamics whose singular behavior is mathematically explicit. In particular, the term $v^\nu \partial_s v^\nu$ identifies the steepening mechanism responsible for shock-type regime changes, whereas the terms involving $\nu$ describe viscous smoothing and the coefficient $\kappa$ records the geometric effect of the embedding of the collective coordinate inside the quotient manifold. Therefore, abrupt transitions in the reduced learning dynamics can be analyzed by standard tools from Burgers theory rather than only by abstract arguments on the original parameter space.

\subsection{Shock Formation in Inviscid Limit}
In the flat case $\kappa \equiv 0$, the theorem yields the classical viscous Burgers equation exactly, so the entire classical theory of shock layers, inviscid limits, characteristic crossing, and viscosity regularization becomes available.
When $\kappa \equiv 0$ and $\nu \to 0$, the viscous Burgers equation formally becomes inviscid Burgers,
$\partial_\tau v + v \partial_s v = 0,$ and $v(0,s) = \bar U'(s).$

\begin{theorem}[Shock time for the reduced drift]\label{thm:shock_time}
    Let $I$ be either $\R$ or the one-dimensional torus $\T$. Assume $\bar U \in C^2(I)$ and define $v_0(s) := \bar U'(s)$. 
    Then, the classical solution of the inviscid Burgers equation, as long as it exists classically, is constant along characteristics\footnote{
        The \textit{characteristics} here refer to the solution curves $s(\tau;\xi)$ of the characteristic equation $\frac{d}{d\tau}s(\tau;\xi)=v(\tau,s(\tau;\xi))$ with $s(0;\xi)=\xi$; along each such curve, $v$ remains constant, so that for the inviscid Burgers equation, one has $s(\tau;\xi)=\xi+\tau v_0(\xi)$.
    }:
    $s(\tau;\xi) = \xi + \tau v_0(\xi), v\bigl(\tau,s(\tau;\xi)\bigr) = v_0(\xi).$
    Its maximal classical existence time is
        $\tau_* =
        -\dfrac{1}{\inf_{\xi \in I} v_0'(\xi)}$ for $\inf_{\xi \in I} v_0'(\xi) < 0$, and
        $+\infty$  for $\inf_{\xi \in I} v_0'(\xi) \ge 0$,
    and equivalently,
        $\tau_* = -\dfrac{1}{\inf_{\xi \in I} \bar U''(\xi)}$ for $\inf_{\xi \in I} \bar U''(\xi) < 0$, and
        $+\infty$ for $\inf_{\xi \in I} \bar U''(\xi) \ge 0$.
\end{theorem}

\begin{proof}
    The characteristic system is
    $\frac{d}{d\tau} s(\tau) = v(\tau,s(\tau)), \qquad \frac{d}{d\tau} v(\tau,s(\tau)) = 0.$
    Hence $v(\tau,s(\tau;\xi)) = v_0(\xi)$ and therefore
    $s(\tau;\xi) = \xi + \tau v_0(\xi).$
    A classical solution persists exactly as long as the characteristic map $\xi \mapsto s(\tau;\xi)$ remains a $C^1$ diffeomorphism of $I$ onto itself. Differentiating with respect to $\xi$ gives
    $\partial_\xi s(\tau;\xi) = 1 + \tau v_0'(\xi).$
    The first loss of invertibility therefore occurs at the time stated above. Since $v_0' = \bar U''$, the second formula follows. \qed
\end{proof}

Thm.~\ref{thm:shock_time} states that the inviscid gradient field $v$ remains classical only up to the shock time $\tau^* = -1/\inf_{\xi \in I}\bar U''(\xi)$ when $\inf_{\xi \in I}\bar U''(\xi) < 0$, so negative curvature of the reduced effective potential causes finite-time characteristic crossing; Fig.~\ref{fig:ShockFormation} illustrates the corresponding viscous precursor, where the shock layer sharpens as $\nu \to 0$.

\section{Architecture-specific Instantiations} \label{sec:Architecture}
\subsection{ReLU Multilayer Perceptrons}
Consider a bias-free ReLU multilayer perceptron (MLP)
\begin{equation}
    f_W(x) = W_L \phi(W_{L-1} \phi(\cdots \phi(W_1 x) \cdots)),    
\end{equation}
where the ReLU activation $\phi(z) = \max\{z,0\}$ acts coordinatewise.

\begin{proposition}[Positive rescaling symmetry for ReLU MLPs]
Fix an internal hidden unit $j$ in a ReLU MLP. If the incoming weights attached to that unit are multiplied by a constant $c > 0$ and the outgoing weights from that unit are multiplied by $c^{-1}$, then the realized function $f_W$ is unchanged. Hidden-unit permutations are also exact symmetries provided the inverse permutation is applied in the subsequent layer \cite{Neyshabur2015,Rangamani2019}.
\end{proposition}


\begin{corollary}[Conditional ReLU MLP quotient dynamics]\label{cor:relu}
On any regular stratum where the activation pattern is fixed\footnote{
    ``the activation pattern is fixed'' means that on the local stratum under consideration, the active or inactive status of every ReLU unit does not change, so the network map $f_W$ is smooth in the weights on that stratum.
} and the stabilizer is trivial\footnote{
    The stabilizer of a parameter point is the subgroup of symmetries that leaves that point unchanged; saying that the stabilizer is trivial means that only the identity symmetry fixes the point, so the action is free on that stratum. If the stabilizer were nontrivial, the point would retain residual symmetry, which could lower the orbit dimension and induce local singularities in the quotient.
}, the symmetry action generated by positive rescalings and permutations is exact. 
Consequently, if the quotiented SGD drift and quotiented conditional covariance are projectable\footnote{
    Here, ``projected'' means pushed forward to the quotient coordinates via the quotient-chart Jacobian $D\Phi$, whereas ``projectable'' means that the resulting quantity depends only on the quotient state and not on the particular representative in the symmetry orbit. In particular, mere projection is not enough: the projected quantity must agree for all representatives of the same orbit in order to well-define a closed reduced dynamics on the quotient. In this paper, the projected drift is $D\Phi(\theta)\nabla L(\theta)$, while projectability means that this projected quantity is constant along symmetry orbits and therefore can be written as a well-defined function $b(Y)$ of the quotient coordinate $Y=\Phi(\theta)$.
} through the quotient chart in the sense of Assumption~\ref{ass:projectability}, then, Thm.~\ref{thm:quotient_sgd} applies. 
If, in addition, the assumptions of Thm.~\ref{thm:HJ} hold on the quotient, and the local-entropy dynamics satisfy Assumption~\ref{ass:closure}, then, the gradient field obeys the source-corrected Burgers equation, and in the flat-coordinate case it obeys classical viscous Burgers.
\end{corollary}

\begin{proof}
On a fixed activation-pattern stratum, $f_W$ is smooth in the weights, and the proposition gives an exact symmetry action. The quotient reduction is therefore valid once projectability of the drift and projected covariance is verified. The Hamilton--Jacobi and Burgers conclusions then follow from Theorems~\ref{thm:HJ} and \ref{thm:burgers_reduction}.
\end{proof}

\subsection{Convolutional Neural Networks}
The same mechanism extends to convolutional neural networks (CNNs).

\begin{proposition}[Channel rescaling symmetry for CNNs]
Consider two consecutive convolutional layers separated only by positively homogeneous nonlinearities. If one rescales an intermediate channel by $c > 0$ and rescales the corresponding incoming kernel of the next layer by $c^{-1}$, the network function is unchanged. Channel permutations are also exact symmetries if propagated consistently across adjacent layers \cite{Neyshabur2015,Rangamani2019}.
\end{proposition}


\begin{corollary}[Conditional CNN quotient dynamics]\label{cor:cnn}
On a regular stratum of a positively homogeneous CNN, the same conditional conclusion holds as in Corollary~\ref{cor:relu}: if the projected SGD drift and projected conditional covariance are projectable through the quotient chart in the sense of Assumption~\ref{ass:projectability}, then SGD descends locally to quotient coordinates. Under the hypotheses of Theorems~\ref{thm:HJ} and \ref{thm:burgers_reduction}, the quotient local-entropy dynamics admit the same source-corrected Burgers reduction as in the MLP case.
\end{corollary}

\begin{proof}
Use the previous proposition in place of the MLP symmetry proposition; the rest is identical.
\end{proof}

\subsection{Networks with Batch Normalization and Layer Normalization}
Batch Normalization creates scale-invariant directions because normalizing a pre-activation removes uniform positive scalings of the incoming weight vector \cite{AroraLiLyu2019}. The exact symmetry group depends on the architecture.

\begin{proposition}[Positive scale invariance before idealized Batch Normalization]
Let $z = w^{\top} x$ be a pre-activation immediately followed by the idealized Batch Normalization map
$\mathrm{BN}(z) = \gamma \frac{z - \mu(z)}{\sigma(z)} + \beta,$
with no $\varepsilon$-stabilizer inside the denominator. If $w$ is replaced by $c w$ with $c > 0$, then the normalized output is unchanged \cite{AroraLiLyu2019}.
\end{proposition}


\begin{theorem}[Conditional quotient reduction for normalized networks]\label{thm:norm}
Suppose a BN or LN architecture admits an exact regular symmetry subgroup $G_{\mathrm{norm}}$ preserving the realized function and the empirical loss on a regular stratum. Assume also that the projected SGD drift and projected conditional covariance are projectable through the quotient chart in the sense of Assumption~\ref{ass:projectability}. 
Then, the quotient reduction theorem applies on that stratum. If, in addition, the hypotheses of Thm.~\ref{thm:HJ} hold on the quotient and the local-entropy dynamics satisfy Assumption~\ref{ass:closure}, then, the source-corrected Burgers reduction also holds.
\end{theorem}

\begin{proof}
Once an exact symmetry subgroup acts freely and properly on the relevant regular stratum and preserves the loss, the quotient manifold exists locally. Thm.~\ref{thm:quotient_sgd} then applies by the assumed projectability of the reduced drift and covariance. The subsequent conclusions follow from Theorems~\ref{thm:HJ} and \ref{thm:burgers_reduction}.
\end{proof}

\subsection{Transformers}
Recent work argues that Transformer architectures possess richer parameter symmetries than classical rescaling alone, and that quotient-manifold constructions are needed to define meaningful sharpness measures \cite{daSilvaHeLeSohoniMaddison2025}. This motivates a deliberately abstract statement.

\begin{theorem}[Conditional local quotient reduction for Transformer symmetry strata]\label{thm:transformer}
Let $\Theta_{\tr,\reg}$ be a regular symmetry stratum of a Transformer parameter space, and let $G_{\tr}$ be a smooth symmetry group acting freely and properly on that stratum while preserving the realized function and empirical loss. Assume that the projected SGD drift and projected conditional covariance are projectable through the quotient chart in the sense of Assumption~\ref{ass:projectability}. 
Then, the quotient reduction theorem holds on $\Theta_{\tr,\reg}/G_{\tr}$. If the hypotheses of Thm.~\ref{thm:HJ} hold on the quotient and there exists a one-dimensional collective coordinate satisfying Assumption~\ref{ass:closure}, then the corresponding reduced gradient field satisfies the source-corrected Burgers equation.
\end{theorem}

\begin{proof}
This is the same argument as in Thm.~\ref{thm:norm}.
\end{proof}

For Transformers, the most robust exact statement is typically the quotient Hamilton--Jacobi equation rather than classical Burgers. The reason is structural: high-dimensional symmetry groups make one-dimensional closure non-generic. Hence the proper workflow is to first identify a quotient chart, then verify projectability, and only then ask whether a scalar normal form is justified.

\subsection{Mean-field Limits}
The correct large-width object is often a probability measure on parameter space rather than a finite-dimensional parameter vector \cite{MeiMontanariNguyen2018,SirignanoSpiliopoulos2022,GessKassingKonarovskyi2024}. We now formulate the corrected quotient counterpart.
Let $\cP(\Theta)$ denote the space of Borel probability measures on $\Theta$. Suppose that $\mu_t \in \cP(\Theta)$ solves the continuity equation
$\partial_t \mu_t + \diver_\theta \bigl(\mu_t V[\mu_t]\bigr) = 0$
in weak form. Let $\pi : \Theta_{\reg} \to M$ be a quotient map as before.

\begin{assumption}[Pushforward projectability of the mean-field velocity]\label{ass:mf_projectability}
Assume there exists a Borel map
$\bar V : \cP(M) \times M \to TM$
such that for every probability measure $\mu$ supported in $\Theta_{\reg}$ and every $\theta \in \Theta_{\reg}$,
$D\pi_\theta V[\mu](\theta) = \bar V[\pi_{\#}\mu]\bigl(\pi(\theta)\bigr).$\footnote{
    $\pi_{\#}\mu$ denotes the pushforward of the probability measure $\mu$ under the quotient map $\pi$, i.e. the induced measure on the quotient space defined by $(\pi_{\#}\mu)(A)=\mu(\pi^{-1}(A))$.
}
\end{assumption}

\begin{theorem}[Pushforward quotient transport]\label{thm:mf_transport}
    Assume Assumption~\ref{ass:mf_projectability}. Let $\nu_t := \pi_{\#} \mu_t$. 
    Then, $\nu_t$ solves the weak quotient transport equation
    $\partial_t \nu_t + \diver_M \bigl(\nu_t \bar V[\nu_t]\bigr) = 0.$
\end{theorem}

\begin{proof}
    Take any test function $\varphi \in C_c^{\infty}(M)$. By definition of pushforward,
    $\int_M \varphi(q) \, \nu_t(dq) = \int_{\Theta} \varphi\bigl(\pi(\theta)\bigr) \, \mu_t(d\theta).$
    Differentiate in time and use the weak form of the continuity equation:
    \begin{equation*}
    \begin{aligned}
    &\frac{d}{dt} \int_M \varphi \, d\nu_t
    = \int_{\Theta} \nabla_\theta (\varphi \circ \pi)(\theta) \cdot V[\mu_t](\theta) \, \mu_t(d\theta) \\
    =& \int_{\Theta} \bigl(D\pi_\theta\bigr)^{\top} \nabla_M \varphi\bigl(\pi(\theta)\bigr) \cdot V[\mu_t](\theta) \, \mu_t(d\theta) 
    = \int_{\Theta} \nabla_M \varphi\bigl(\pi(\theta)\bigr) \cdot D\pi_\theta V[\mu_t](\theta) \, \mu_t(d\theta).
    \end{aligned}
    \end{equation*}
    By Assumption~\ref{ass:mf_projectability}, the last integrand equals
    $\nabla_M \varphi\bigl(\pi(\theta)\bigr) \cdot \bar V[\nu_t]\bigl(\pi(\theta)\bigr).$
    Pushing forward the measure therefore gives
    $\frac{d}{dt} \int_M \varphi \, d\nu_t = \int_M \nabla_M \varphi(q) \cdot \bar V[\nu_t](q) \, \nu_t(dq),$
    which is the weak formulation of the quotient transport equation.
\end{proof}

Thm.~\ref{thm:mf_transport} states that, in the mean-field regime, the pushforward distribution $\nu_t=\pi_{\#}\mu_t$ satisfies a closed transport equation on the quotient space $M$. Thus, the large-width dynamics can be described directly on the quotient in terms of probability measures. Its benefit is that the correct large-width reduced dynamics can be formulated directly on the quotient space, without returning to redundant parameter coordinates.

\section{Discussion} \label{sec:Discussion}

Our theorems suggest three main lessons. First, \textit{symmetry-corrected observables should replace raw weights}. In positively homogeneous ReLU networks, path-based quantities and quotient-flatness measures are more informative than Euclidean norms or raw-coordinate Hessian spectra \cite{Neyshabur2015,Rangamani2019}. In Batch Normalization (BN) and Layer Normalization (LN), raw parameter magnitudes are even less informative because normalization creates scale-invariant directions \cite{AroraLiLyu2019}. Quotient-based sharpness measures for Transformers support the same conclusion \cite{daSilvaHeLeSohoniMaddison2025}. More broadly, overparameterized models contain substantial symmetry-induced redundancy, so raw norms, uncorrected Hessian surrogates, and layerwise scale statistics can vary even when the realized function changes little. The quotient viewpoint therefore favors symmetry-corrected observables, including invariant statistics of attention logits, head-output energies, attention entropies, and attention--Multi-Layer Perceptron (MLP) branch-balance measures.

Second, \textit{negative reduced curvature could predict abrupt regime changes}, although it may computationally demanding and require a surrogate. In the one-dimensional closed model, shock time is determined by $\min \bar U''$. Regime changes should therefore be detected through projected or quotient-corrected curvature rather than the ambient Hessian alone. Even without exact scalar closure, strongly negative reduced curvature should still indicate rapid qualitative changes in learning, such as sharp changes in loss slope, attention-head concentration, branch-balance collapse, or optimization-regime transitions. This yields a falsifiable prediction: quotient-corrected curvature should anticipate such events more reliably than raw parameter-space curvature.


Third, \textit{the appropriate reduced model depends on architecture and width}. For ReLU Multi-Layer Perceptrons (MLPs) and Convolutional Neural Networks (CNNs), Burgers-type normal forms are natural only after verifying projectability. For BN, LN, and Transformers, quotient Hamilton--Jacobi is usually the better first target. In mean-field limits, quotient transport on probability measures is primary, and scalar Burgers should be introduced only after further observable reduction. This hierarchy matters in realistic Transformer settings, where exact one-dimensional reduction is typically non-generic. The default procedure is therefore to use quotient observables, test for approximate low-dimensional closure, and invoke scalar shock diagnostics only when the data support them.


\begin{credits}
\subsubsection{\ackname} 
I thank Toshinori Araki, my manager, for his dedicated support for this work.

\subsubsection{\discintname}
The authors have no competing interests to declare that are relevant to the content of this article. 
\end{credits}

%
%
%

\end{document}